\documentclass{article}

\usepackage[nonatbib,main,preprint]{neurips_2026}

\usepackage[numbers]{natbib}

\usepackage[utf8]{inputenc} 
\usepackage[T1]{fontenc}    
\usepackage{hyperref}       
\usepackage{url}            
\usepackage{booktabs}       
\usepackage{amsfonts}       
\usepackage{nicefrac}       
\usepackage{microtype}      
\usepackage{xcolor}         

\usepackage{subcaption}
\usepackage{algorithm}
\usepackage{algorithmic}
\usepackage{color,soul}
\usepackage{tabularx}
\usepackage{multirow}
\usepackage{ragged2e}
\usepackage{tabularx}
\usepackage{arydshln}
\usepackage{amsmath}
\usepackage{amssymb}
\usepackage{mathtools}
\usepackage{amsthm}
\usepackage{tcolorbox}
\usepackage{array}
\usepackage{makecell}
\usepackage{caption}
\usepackage[capitalize,noabbrev]{cleveref}
\usepackage{eqparbox}
\usepackage{graphicx}
\usepackage{inconsolata}
\usepackage{wrapfig}
\usepackage{enumitem}

\newtheorem{proposition}{Proposition}

\theoremstyle{remark} 
\newtheorem{remark}{Remark}

\title{Structural Abstraction as an Inductive Bias for Non-Stationary Language Model Training}

\author{Elnaz Rahmati \And Nona Ghazizadeh \And Zhivar Sourati \And Nina Rouhani \And Morteza Dehghani
        \AND
\normalfont  University of Southern California \\
\texttt{\small\{erahmati, nghaziza, souratih, nrouhani, mdehghan\}@usc.edu} }

\begin{document}

\maketitle

\begin{abstract}
A foundational principle in cognitive science holds that intelligent agents do not learn by storing experiences as isolated instances, but by forming abstract schemas that capture relational structure shared across situations. Even though this claim is well supported by behavioral and neuroimaging studies, its role as a computational training signal in language models remains underexplored. We target this gap in the setting of non-stationary language model training, asking does biasing learning toward structural abstraction reduce catastrophic interference and improve relational generalization as predicted by human results? To study this question, we introduce Abstraction-Augmented Training (AAT), a lightweight loss-level modification that jointly optimizes over concrete instances and their structural abstractions, and two benchmarks, the Relational Cycle Benchmark (RCB) and the Narrative Abstraction Benchmark (NAB). These resources operationalize core cognitive constructs: entity masking as a computational analog of relational alignment, and proverbs as vehicles for implicit abstract meaning that must be inferred across surface-dissimilar situations. Our empirical results demonstrate that AAT consistently reduces forgetting and improves generalization in a pattern that aligns with cognitive predictions for schema-based learning. Beyond the practical implications for continual learning, these results offer preliminary computational evidence that structural abstraction is a signal for stable learning in non-stationary environments.
\end{abstract}

\section{Introduction}
\label{sec:intro}

A central question in cognitive science is how intelligent agents acquire knowledge that transfers across contexts, rather than merely memorizing the specific instances they have encountered. Behavioral and theoretical work suggests that humans do not primarily store experiences as isolated episodes, but instead form abstract \emph{schemas} -- representations that capture relational structure shared across situations \citep{schank2013scripts}. These schemas enable generalization by supporting the application of knowledge learned in one context to unseen yet structurally similar situations \citep{gick1983schema,gentner1983structure}. Human behavioral findings further demonstrate that experiences emphasizing abstract relational structure, rather than instance-specific details, are more likely to be retained and reused in new contexts \citep{clement1994effects}. Cognitive neuroscience corroborates this: when events are encoded in the brain, learning tends to organize experience around components shared across related events, supporting replay at multiple levels of abstraction and enabling stable generalization and decision-making \citep{masis2022schema,mattar2018prioritized}.

These principles become especially consequential in non-stationary settings, where knowledge must be acquired continuously from an evolving stream of information. Language models are pre-trained on data collected up to a fixed point in time, yet real-world information continues to evolve. Retraining from scratch as new data becomes available is computationally prohibitive, motivating \emph{online continual learning} (OCL); a regime in which data arrives as a stream, task boundaries are absent, and each example is observed only once. In this setting, the tension between instance-level memorization and structural generalization is most acute: models must not only learn new facts, but retain previously acquired knowledge and induce patterns that transfer to unobserved cases. Training on non-stationary streams introduces two well-studied challenges: \emph{catastrophic forgetting}, where previously learned knowledge is overwritten \citep{goodfellow2015empiricalinvestigationcatastrophicforgetting,doi:10.1073/pnas.1611835114,nguyen2019toward,pmlr-v162-mirzadeh22a,tadros2022sleep,11151751}, and \emph{plasticity loss}, where the model's capacity to incorporate new information degrades over time \citep{pmlr-v202-lyle23b,dohare2024loss,ICLR2025_359c0c4d,ICLR2025_4d423587,ICLR2025_ba9e3d60}.

Existing methods address these challenges through regularization \citep{doi:10.1073/pnas.1611835114,chaudhry2018efficient} or experience replay (ER), which revisits stored examples from earlier training to stabilize learning \citep{NEURIPS2019_fa7cdfad}. While ER is often effective \citep{van2020brain,NEURIPS2020_b704ea2c,wang2025experience}, it achieves stability by re-exposing the model to past instances requiring a memory buffer, which becomes increasingly costly in strictly online settings where storage is limited or infeasible. Importantly, such instance-level replay does not address the deeper problem that models trained primarily on surface-level instance signals may successfully recall training examples while losing the shared relational structures necessary for inductive reasoning.

Motivated by the cognitive principle that abstract, schema-consistent representations are more robustly retained and generalized \citep{hofstadter2001analogy,gentner2003we}, we introduce \emph{Abstraction-Augmented Training} (AAT), a loss-level inductive bias. AAT jointly optimizes over a concrete instance and its corresponding structural abstraction at first encounter, biasing the model toward representations invariant to instance-specific surface variation. A small number of local replays then reinforce the concrete instance, consolidating factual details without allowing entity-specific gradients to dominate. Critically, AAT stores no past data and its stabilizing effect arises entirely from the structural signal introduced during learning.

To study the role of abstraction rigorously, we introduce two benchmarks that operationalize the distinction between \emph{episodic} retention of specific facts and \emph{structural} generalization to unobserved relations. The \emph{Relational Cycle Benchmark} (RCB) is derived from knowledge graphs and instantiates abstraction through entity masking, following the logic of Structure-Mapping Theory (SMT)\citep{gentner1983structure}. SMT suggests that removing surface-level entity identities isolates the relational skeleton that supports analogical transfer. This is complemented by the \emph{Narrative Abstraction Benchmark} (NAB), which operates at a higher level by requiring models to align narratives through shared proverbs -- standard proxies for abstract relational meaning \citep{gentner1983structure,honeck1997proverb} -- even when surface forms differ entirely.

We provide a theoretical account showing how AAT suppresses instance-specific gradient components while amplifying updates aligned with shared relational structure, paralleling the cognitive finding that schema-congruent information is preferentially encoded and retained \citep{masis2022schema}. Our empirical results across both benchmarks and two model families show that AAT produces a characteristic pattern predicted by cognitive science: reduced forgetting, improved generalization to unobserved relational structure, and greater robustness to data order.

Taken together, this work makes three contributions: (i) we introduce RCB and NAB as benchmarks that operationalize cognitive constructs as measurable properties of language model behavior, filling a gap in OCL evaluation while offering tools for studying schema-based learning; (ii) we introduce AAT as a method for testing whether structural abstraction, as a training signal, is sufficient to produce schema-consistent learning dynamics; (iii) we provide empirical and theoretical evidence that it is, and interpret these findings not only as a practical finding in OCL, but as a step toward understanding the computational basis of abstract relational learning in natural and artificial systems.

\section{Method}

Cognitive accounts of human learning distinguish between \emph{episodic} representations, which encode specific experienced events, and \emph{schematic} representations, which capture relational patterns shared across events \citep{schank2013scripts}. Effective long-term learning requires both: episodic detail supports recall of specific facts, while schematic structure supports generalization and reduces interference between related memories \citep{masis2022schema}. This tension between instance-level memorization and structural generalization also arises in machine learning. We operationalize this distinction computationally within OCL. In the following sections we introduce AAT and provide an analysis of its stabilizing effect.

\subsection{Abstraction-Augmented Online Learning}

A well-documented failure mode in human learning occurs when surface-level details overwhelm the underlying relational structure, leading to context-bound representations that impair transfer to new domains \citep{gick1983schema,clement1994effects}. A similar pattern arises when fine-tuning LLMs on knowledge-graph-derived data streams: entity-specific surface features tend to dominate the learning signal, overwriting parameters that encode broader relational dependencies \citep{jang2021towards,wu2023online}. As a result, models may successfully recall training instances while losing the shared structures necessary for inductive reasoning. 

\begin{wrapfigure}{r}{0.47\textwidth}
  \centering
  \hrule height 1pt
  \vspace{3pt}
  \captionof{algorithm}{Abstraction-Augmented Training}
  \label{alg:abstraction}
  \vspace{-3pt}
  \hrule 
  \vspace{4pt}
  
  \begin{footnotesize} 
  \renewcommand{\algorithmiccomment}[1]{ {\footnotesize \textcolor{gray}{\textsl{// #1}}}} 
  
  \begin{algorithmic}[1]
     \STATE {\bfseries Input:} Dataset $\mathcal{D}$, initial parameters $\theta$, abstraction ratio $\alpha$, local replay count $n$, learning rate $\eta$.
     \FOR{each batch $\mathcal{B} \subset \mathcal{D}$}
        \FOR{$j = 0$ {\bfseries to} $n-1$}
           \STATE $\hat{y}_{inst} \leftarrow \pi(x_{inst}; \theta)$ \COMMENT{Instance forward pass}
           \STATE $\hat{y}_{abs} \leftarrow \pi(x_{abs}; \theta)$ \COMMENT{Abstract forward pass}
           
           \IF{$j = 0$} \label{alg:if-j-0}
              \STATE $\mathcal{L}(\theta) \leftarrow \alpha \cdot \text{NLL}(\hat{y}_{abs}, y_{abs}) + (1 - \alpha) \cdot \text{NLL}(\hat{y}_{inst}, y_{inst})$ \COMMENT{Combined loss}
           \ELSE \label{alg:if-j>0}
              \STATE $\mathcal{L}(\theta) \leftarrow \text{NLL}(\hat{y}_{inst}, y_{inst})$ \COMMENT{Instance loss}
           \ENDIF
           
           \STATE $\theta \leftarrow \text{Adam}(\nabla_\theta \mathcal{L}(\theta), \theta, \eta)$ \COMMENT{Parameters update}
        \ENDFOR
     \ENDFOR
  \end{algorithmic}
  \end{footnotesize}
  
  \vspace{5pt}
  \hrule
\end{wrapfigure}

Formally, we consider a data stream $D = \{d_1, \dots, d_m\}$ where each instance is observed only once. At each training step $i$, the model $\pi_i$ encounters a batch $b_i = \{d^i_1, \dots, d^i_B\}$. The objective is to maximize performance on the current batch $b_i$ while minimizing forgetting on the historical data $\bigcup_{k=1}^{i-1}b_k$.
We draw on SMT and theory-based categorization \citep{murphy1985role}, which propose that analogical reasoning proceeds by aligning relational structure across domains while abstracting away from surface-level attributes. Computationally, we interpret this as a gradient-level inductive bias. By co-optimizing over a concrete instance and its structural abstraction, we suppress entity-specific gradient components and amplify updates aligned with shared relational patterns, thereby reducing catastrophic interference across sequential examples.

AAT integrates this via a dual-objective loss function paired with local replay (see \Cref{alg:abstraction}). In OCL, each data batch is encountered exactly once. Because information-dense batches may not be fully consolidated in a single update, we introduce a local replay count $n \in [1, 5]$, allowing multiple optimization steps per batch. On first exposure, the model receives both a concrete instance ($x_{inst}$) and its corresponding structural abstraction ($x_{abs}$), where the abstraction is constructed by removing instance-specific surface features (e.g., masking entity identities). This biases the update toward representations invariant to surface variation. The abstraction's contribution to the loss is controlled by $\alpha$, defaulting to $0.5$. Subsequent $n-1$ local replays reinforce the concrete instance, retaining factual detail while the abstract in $j = 0$ prevents entity-specific gradients from dominating. 

\subsection{Theoretical Motivation}

We provide an analysis of why abstraction stabilizes learning under a simplified model. While the assumptions are idealizations, they isolate the mechanism we believe drives the empirical results.

\textbf{Setup.} Consider a model that predicts relations from an input paragraph $x$ using an encoder $\phi_\theta(\cdot)$ followed by a linear head $W$: $f_\theta(x) = W \phi_\theta(x)$. Each example contains two sources of information: (i) relational structure $r$ and (ii) entity identities $e_1, \ldots, e_k$. We decompose the representation as $\phi_\theta(x) = \phi_r(r) + \phi_e(e_1, \ldots, e_k)$, where $\phi_r$ encodes structural features shared across examples within a relational pattern, and $\phi_e$ encodes instance-specific entity signals.

\textbf{Assumptions.}
\begin{enumerate}[label=\textbf{A\arabic*.}]
    \item \textbf{(Additive decomposition)} The encoder decomposes additively as 
    $\phi_\theta(x) = \phi_r(r) + \phi_e(e_1,\ldots,e_k)$, with $\phi_r$ and 
    $\phi_e$ operating on disjoint parameter subsets.
    \item \textbf{(Asymmetric gradient variance)} Entity tokens are long-tail and carry high per-token loss that varies dramatically across sequential examples, while relation tokens are high-probability words with low, stable per-token loss. Formally:
    \[
    \sigma_e^2 = \mathrm{Var}(\nabla_\theta \phi_e^{(i)}) \gg 
    \sigma_r^2 = \mathrm{Var}(\nabla_\theta \phi_r^{(i)}) \geq 0.
    \]
    \item \textbf{(Entity masking)} Under entity masking, the abstract input removes entity identity, yielding $\phi_\theta(x_{abs}) = \phi_r(r)$, so the abstract forward pass contributes no entity-specific gradient.
\end{enumerate}

\begin{proposition}[Abstraction reduces gradient variance]
\label{prop:gradient}
Under Assumptions \textbf{A1--A3}, the AAT combined loss with weight $\alpha \in (0,1)$,
$
\mathcal{L}_i^{\mathrm{total}} = \alpha \cdot \mathcal{L}_i^{\mathrm{abs}} + 
(1-\alpha) \cdot \mathcal{L}_i^{\mathrm{inst}}
$
, is a weighted combination of a low-variance signal (abstract loss, contributing $\nabla_\theta \phi_r^{(i)}$) and a high-variance signal (instance loss, contributing $\nabla_\theta \phi_r^{(i)} + \nabla_\theta \phi_e^{(i)}$). The variance of the combined gradient is therefore strictly lower than that of the instance-only gradient for any $\alpha > 0$, $\mathrm{Var}\!\left(\nabla_\theta \mathcal{L}_i^{\mathrm{total}}\right) < 
\mathrm{Var}\!\left(\nabla_\theta \mathcal{L}_i^{\mathrm{inst}}\right),$
and the signal-to-noise ratio of the relational gradient increases monotonically with $\alpha$,
$
\frac{\|\nabla_\theta \phi_r^{(i)}\|}{\|(1-\alpha)\nabla_\theta \phi_e^{(i)}\|} \uparrow \alpha.
$

\end{proposition}
\noindent Proof of \Cref{prop:gradient} is provided in Appendix~\ref{appendix-proof}.


\begin{remark}
Assumption \textbf{A1} is the strongest idealization. In practice, Transformers do not maintain disjoint parameter subsets for relational and entity representations. The proposition should therefore be read as isolating a mechanism rather than characterizing transformer behavior precisely. The results in \Cref{sec:results} provide empirical support that the predicted effect holds approximately in practice.
\end{remark}

\noindent The key insight is that abstraction suppresses entity-specific gradient components, which vary dramatically across sequential examples, while amplifying updates aligned with shared relational structure, which is consistent across the data stream. This directly parallels the cognitive finding that schema-congruent information is preferentially encoded and retained \citep{masis2022schema}, with $\nabla_\theta \phi_r^{(i)}$ playing the role of schema-consistent signal. By collapsing many entity-specific examples into a shared relational template, each update reinforces a broader equivalence class of relational patterns acting as an implicit structural regularizer that preserves previously learned relational knowledge even as new entity instances arrive. Unlike ER, which achieves stability by revisiting past instances, AAT achieves a stabilizing effect by reinforcing relational structure itself.

\section{Benchmarks}

Schema-based learning theory distinguishes between two different aspects of knowledge, episodic retention of specific observed facts and structural generalization to unobserved relations inferable from shared patterns \citep{schank2013scripts, gentner1983structure}. A model that has acquired abstract relational structure should exhibit both. We introduce two benchmarks that operationalize this distinction at different granularities, each grounded in a specific cognitive construct. RCB targets relational alignment via entity masking. Removing surface-level entity identities isolates the relational skeleton that supports analogical transfer, in direct computational analogy to the alignment process proposed by \citet{gentner1983structure} as the basis of human analogical reasoning. NAB targets implicit analogical transfer through proverbs \citep{honeck1997proverb}, requiring models to align narratives through shared motifs even when surface form, entities, and event structures differ entirely. These benchmarks provide a computational testbed for asking whether language models exhibit the behavioral signatures of schema-based learning.

\paragraph{Relational Cycle Benchmark.} RCB jointly probes factual retention and deductive generalization. It is constructed from a curated ``Relation Bank'' of $\{\text{entity}_1, \text{relation}, \text{entity}_2\}$ triples spanning eight semantic domains. We define 51 distinct relational typologies, small relational subgraphs with a fixed structure. Each typology contains at least three relations arranged to form an undirected cycle. In 60\% of typologies, the cycle encodes an explicit logical dependency in which one relation is strictly entailed by the conjunction of the others (e.g., $\{\textsc{Mother}(\text{Sarah}, \text{John}), \textsc{Mother}(\text{Sarah}, \text{Jane})\}$ entails $\textsc{Sibling}(\text{John}, \text{Jane})$). For remaining typologies, where no entailment is identified under our rule set, one edge is randomly selected as held-out. Each typology is populated with approximately 25 instances by instantiating the same relational structure with entirely different entities sourced from Wikidata \citep{vrandevcic2014wikidata}, yielding 1,245 typology instances and 3,295 unique entity-relation triples. 

For each instance, the held-out relation is treated as the \emph{unknown edge}, while remaining relations (\emph{known edges}) are serialized into a natural-language paragraph with lexical variation to reduce template overfitting. These paragraphs include distractor relations not directly relevant to recovering the unknown edge, requiring the model to identify and apply the relevant structural pattern rather than rely on surface co-occurrence. The design of RCB reflects a specific cognitive hypothesis; the distinction between known and unknown edges maps onto the distinction between episodic and schematic memory \citep{schank2013scripts}. Known edges correspond to directly experienced facts; unknown edges correspond to relations that are inferrable only if the model has abstracted the relational structure underlying the observed instances, the computational analog of schema-supported inference in human memory. (See Appendix~\ref{appendix:datastats} for statistics and baseline performance). 


\paragraph{Narrative Abstraction Benchmark.} NAB probes a higher-order form of abstraction: the alignment of narratives through shared relational motifs expressed as proverbs, even when surface form, entities, and event structures differ entirely. Proverbs have a long history in the analogical reasoning literature as vehicles for abstract relational meaning that must be applied across surface-dissimilar situations \citep{gentner2001metaphor,honeck1997proverb}, making them a natural operationalization of abstraction. Unlike RCB, where entity masking provides an explicit structural signal, NAB requires models to identify implicit relational regularities, without lexical or surface-level overlap between training and evaluation instances.  

Behavioral work shows humans align surface-dissimilar situations by mapping relational roles rather than surface features, requiring genuine abstraction. NAB tests whether language models exhibit similar sensitivity to these proverb-level motifs. The benchmark is constructed from narratives paired with proverbs from ePiC \citep{ghosh-srivastava-2022-epic}, expanded with 15 additional narratives per proverb generated via a prompt-guided analogical generation procedure informed by SMT. This procedure encourages domain-level shifts and penalizes surface overlap, producing narratives that share abstract relational meaning while differing substantially in entities, events, and surface realization (see Appendix~\ref{appendix:epic_expansion}). 

Each narrative is then converted into a pairwise continuation comparison task: a \emph{correct ending} consistent with the proverb's abstract meaning, and a \emph{distractor ending} that is locally coherent but violates the underlying relational motif. Distractors are length-matched and constructed to be plausible story continuations, ensuring that correct prediction requires sensitivity to the proverb-level abstraction rather than superficial coherence cues (see Appendix~\ref{appendix:epic_distractor}). Zero-shot performance on NAB across Qwen2.5-Instruct \citep{qwen2.5} model sizes are provided in Appendix~\ref{appendix:datastats}.

\section{Experimental Set-Up}

We evaluate AAT in an OCL setting, focusing on the trade-off between episodic retention and structural generalization as learning progresses. This section describes the experimental protocol, baselines, and evaluation metrics used to assess learning dynamics and final performance. Unless otherwise stated, all models are trained under identical optimization and data stream conditions.

\subsection{Baselines and Evaluation}\label{sec:baselines}

AAT's combined loss (Algorithm~\ref{alg:abstraction}) introduces a structural stabilization term alongside the instance objective. At the moment of first exposure, this term can be instantiated in two qualitatively different ways: as a structural abstraction of the current instance ($x_{abs}$, as in AAT), or as a sample drawn from a buffer of previously observed instances (as in ER). Both choices target the same optimization problem, reducing interference by supplementing the instance gradient with a signal that generalizes across examples, but through fundamentally different mechanisms. AAT exploits shared relational structure within the current example; ER exploits temporal coverage across past examples. Comparing them directly reveals what each mechanism contributes to stability and generalization. As a lower-bound reference, we include standard fine-tuning with the instance-level loss only ($\alpha = 0$), which isolates the effect of the structural stabilization term by removing it entirely. For ER, we use a reservoir buffer of size 100 with local replay enabled, representing the strongest replay configuration. See Appendix~\ref{sec:imp-detail} for more information on implementation and model selection.

We adopt an interval-based evaluation protocol, assessing the model every $k$ training steps. At evaluation step $i$, metrics are computed over the cumulative batches since the last evaluation, $\mathcal{B}_{cur} = \bigcup_{j=i-k+1}^{i} b_j$, and the full training history $\mathcal{H}_i = \bigcup_{j=1}^{i-k} b_j$. Correctness in RCB is defined by the model's predictive confidence: given a triple $(e_1, r, e_2)$ serialized into a natural language prompt, an instance is correct if $\pi_i(e_2 \mid e_1, r) > \tau$. For NAB, an instance is correct if the conditional log-likelihood of the correct ending exceeds that of the distractor given the same narrative context. We define $\mathbb{I}(\pi_i(d)) = 1$ if the prediction is correct and 0 otherwise.

\textbf{Online Accuracy ($\text{Acc}_{\text{online}}$).} Measures model plasticity, acquiring recently encountered information:
\begin{equation}
\text{Acc}_{\text{online}}(i) = \frac{1}{|\mathcal{B}_{cur}|} \sum_{d \in \mathcal{B}_{cur}} \mathbb{I}(\pi_i(d))
\end{equation}

\textbf{Cumulative Accuracy ($\text{Acc}_{\text{cumul}}$).} Provides a global view of the model's knowledge base, capturing backward transfer as learning new relational structures may clarify previously forgotten edges:
\begin{equation}
\text{Acc}_{\text{cumul}}(i) = \frac{1}{|\bigcup_{j=1}^{i} b_j|} \sum_{d \in \bigcup_{j=1}^{i} b_j} \mathbb{I}(\pi_i(d))
\end{equation}

\textbf{Forgetting ($\text{F}$).} Measures stability by tracking loss of previously acquired knowledge. An instance is forgotten if correctly predicted at some prior step $t < i$ but incorrectly predicted at step $i$, normalized by the set of instances the model learned at least once:
\begin{equation}
F(i) = \frac{|\{d \in \mathcal{H}_i \mid \mathbb{I}(\pi_i(d))=0 \land \exists t < i: \mathbb{I}(\pi_t(d)=1)\}|}{|\{d \in \mathcal{H}_i \mid \exists t < i: \mathbb{I}(\pi_t(d)=1)\}|}
\end{equation}

\subsection{Experiments \& Ablations}

All experiments use a Qwen2.5-1.5B backbone trained with a language modeling objective. Unless otherwise specified, AAT uses $\alpha = 0.5$ and $n = 5$ (implementation details in \Cref{sec:imp-detail}). We first examine the impact of abstraction on  training by tracking forgetting and online accuracy throughout training on RCB, reported separately for known and unknown edges. To empirically investigate the theoretical claim that abstraction reduces gradient variance, we perform a loss landscape analysis using linear interpolation in model weight space \citep{NEURIPS2018_a41b3bb3}, measuring variance and coefficient of variation (CV) along the primary training trajectory and a perpendicular axis. 

To analyze abstraction effectiveness, we ablate the structural abstraction levels: beyond entity masking, we evaluate category-based abstraction (entities replaced with semantic categories such as city, institution, or person) and random abstracts (entity-masked sentences combined without structural coherence) as a control. We also evaluate AAT on NAB to assess whether the cognitive benefits of abstraction generalize from explicit symbolic structure to implicit narrative-level relational motifs. For NAB, we use $n = 1$ and $\alpha = 0.15$; with 90\% of narratives per proverb for training and 10\% held out as unknown. Additional ablations on sensitivity to $\alpha$, $n$, and data order in \Cref{sec:appendix-ablation}.

Finally, we compare AAT against both baselines across RCB metrics: cumulative accuracy ($\text{Acc}_{\text{cumul}}$), step-averaged online accuracy ($\text{Acc}_{\text{online}}$), and forgetting ($F$), measured separately for known and unknown edges. We additionally train SmolLM-1.7B \citep{allal2025smollm2} to assess whether AAT's inductive bias generalizes across model families with substantially different pretraining.

\section{Results} \label{sec:results}

We organize results around three questions. First, what does abstraction do to learning dynamics and loss geometry? Second, what properties of the abstraction drive its effect? Third, how does learning via abstraction compare to learning via instance replay, and what does each mechanism preserve?

\subsection{Training Behavior and Loss Geometry}

As shown in \cref{fig:training-curves}, both normal SFT and AAT exhibit similar high-level dynamics: an initial spike in forgetting at domain transitions, followed by gradual stabilization. For known edges, forgetting decreases and plateaus; for unknown edges, it continues to decrease throughout training. Online accuracy follows a complementary pattern; known-edge accuracy rises then declines while unknown-edge accuracy increases substantially. These dynamics reflect the fundamental tension between fitting instance-level supervision and learning the relational structure that supports inductive generalization.
Despite these shared trends, AAT consistently exhibits lower forgetting and higher online accuracy, particularly on unknown edges and during the early stages of known-edge learning
 suggesting that the structural abstraction signal provides a stabilizing effect.

This stabilization is directly reflected in the geometry of the loss landscape. Compared to normal SFT, AAT reduces loss surface variance by 17.00\% and lowers the coefficient of variation from 0.1549 to 0.1375, confirming the theoretical prediction that abstraction guides optimization toward smoother, lower-interference regions. Along the training trajectory, the abstract loss term improves by an average of 0.243 nats (up to 0.427 nats), while adverse effects on the instance loss remain small (0.038 nats on average). Although absolute loss is marginally higher (2.64\% on average), the resulting landscape exhibits smoother loss and fewer sharp parameter interactions. These results show that the abstraction training signal does not merely add noise or auxiliary supervision, it measurably reshapes the optimization landscape in a direction consistent with reduced catastrophic interference.

To further characterize what abstraction changes about learning, we report mean and standard deviation of cumulative accuracy across three random data orderings for SFT and AAT. Training without abstraction exhibits the highest variance on unknown edges (0.0392), indicating that structural generalization is highly sensitive to the order in which specific instances arrive. AAT achieves higher mean performance while substantially reducing the variance, particularly on unknown edges and overall accuracy (0.4927 $\pm$ 0.0071 vs 0.4674 $\pm$ 0.0197; See Appendix~\ref{sec:appendix-ablation} for full results). By biasing learning toward relational patterns shared across examples rather than the specific sequence of instances, abstraction mitigates order sensitivity.  

These results speak directly to the cognitive predictions that motivated the experimental design. Interference theory in human memory predicts that schema-congruent encoding reduces proactive interference between related memories; that is, learning new instances that share relational structure with previously encoded ones should be less disruptive when that structure is explicitly represented \citep{masis2022schema}. The forgetting and gradient variance results provide a computational parallel: AAT reduces the interference underlying catastrophic forgetting by amplifying shared relational gradients, that mirrors the schema-congruent encoding advantage described in the cognitive literature \citep{masis2022schema}. The data ordering results further corroborate this. Just as schema-based encoding in humans reduces sensitivity to the specific order in which instances are encountered \citep{Sko1988,kuehne2000seql}, AAT reduces the sensitivity of structural generalization to the sequence of training examples. These correspondences do not establish that the underlying mechanisms are the same, but they suggest that the computational principle is doing work that is recognizably similar to what cognitive accounts predict.
\begin{figure*}[t]
    \centering
    \includegraphics[width=\linewidth]{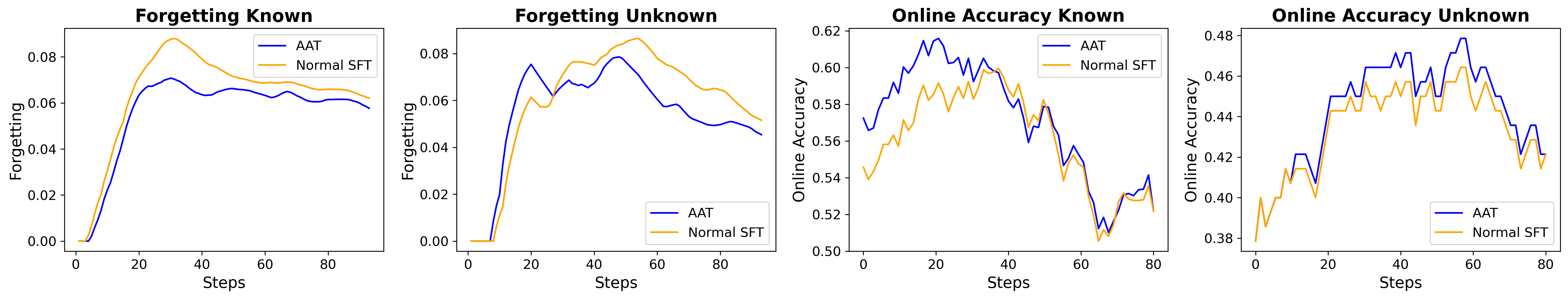}
    \caption{Online accuracy and forgetting over training steps on the RCB.}
    \label{fig:training-curves}
\end{figure*}
\subsection{What Makes Abstraction Effective}






\begin{wraptable}{r}{0.5\textwidth} 
  \centering
      \caption{Performance across abstraction levels.}

  \resizebox{\linewidth}{!}{
    \begin{tabular}{lcccc} 
    \toprule
    \textbf{Method} & $\textbf{Acc}_{\textbf{cumul}}^{\textbf{known}}$ & $\textbf{Acc}_{\textbf{cumul}}^{\textbf{unknown}}$ & $\textbf{F}^{\textbf{known}}$ & $\textbf{F}^{\textbf{unknown}}$ \\
    \midrule
    Random & 0.4756 & 0.3898 & 0.1163 & 0.0845 \\
    Category & 0.5403 & 0.4059 & 0.0713 & 0.0657 \\
    Mask     & 0.5452 & 0.4086 & 0.0679 & 0.0643 \\
    \bottomrule
    \end{tabular}
    
    }
    
    \label{tab:abs-lvl-table}
\end{wraptable}

Having established that abstraction stabilizes learning dynamics, we ask what structural properties of the abstraction drive this effect. \Cref{tab:abs-lvl-table} compares three abstraction levels: entity masking, category-based abstraction (entities replaced with semantic category labels such as city, institution, or person), and random abstracts (entity-masked sentences recombined without preserving relational coherence).

Random abstracts cover the same relations as structured ones during training, yet they degrade performance relative to normal SFT, yielding lower cumulative accuracy and higher forgetting on both edge types. This isolates the critical variable that AAT's benefit does not come from additional training signal, broader relational coverage, or the auxiliary loss itself, but from the structural coherence of the abstraction, whether it preserves the relational skeleton that supports generalization. Incoherent abstracts introduce conflicting gradient signals that increase rather than reduce interference, confirming that the \emph{organization} of the abstract representation, not its mere presence, is what matters.

The gradation random $<$ category $<$ mask maps directly onto predictions from SMT: effective analogical generalization requires suppressing surface attributes sufficiently to expose relational structure, and partial suppression yields proportionally partial benefit. Category-based abstraction retains coarse entity-level information while entity masking removes surface identity entirely. The consistent ordering across all metrics confirms that the degree of structural abstraction, not the type of auxiliary supervision, governs the learning benefit.

\subsection{Generalization to Alternative Abstractions}\label{sec:alt-abstractions}

To assess whether these findings generalize beyond explicitly structured relational graphs, \Cref{tab:epic-table} reports performance on the NAB dataset, where abstraction is implicit and expressed through shared proverb-level meaning across surface-dissimilar narratives. While instance replay achieves the highest accuracy on known narratives, the forgetting results are more revealing. AAT is the only method that reduces forgetting on both known and unknown narratives, whereas instance replay, despite its stronger known-narrative retention, substantially increases forgetting on unknown narratives relative to standard fine-tuning (0.1730 vs. 0.1460), suggesting that revisiting past instances actively interferes with implicit relational generalization. AAT attains the highest accuracy on unknown narratives alongside the lowest forgetting on both dimensions, a profile consistent with its inductive bias toward shared relational 
structure rather than episodic detail.





\begin{wraptable}{r}{0.5\textwidth} 
  \centering
      \caption{Performance comparison on the NAB.}

  \resizebox{\linewidth}{!}{
    \begin{tabular}{lcccc}
    \toprule
    \textbf{Method} & $\textbf{Acc}_{\textbf{cumul}}^{\textbf{known}}$ & $\textbf{Acc}_{\textbf{cumul}}^{\textbf{unknown}}$ & $\textbf{F}^{\textbf{known}}$ & $\textbf{F}^{\textbf{unknown}}$ \\
    \midrule
    Normal SFT & 0.8674 & 0.6086 & 0.0565 & 0.1460 \\
    Buffer-100    & \textbf{0.8940} & 0.5965 & 0.0631 & 0.1730 \\
    AAT     & 0.8683 & \textbf{0.6143} & \textbf{0.0557} & \textbf{0.1429} \\
    \bottomrule
    \end{tabular}
    }
    \label{tab:epic-table}
\end{wraptable}

The NAB result is significant precisely because the abstraction cannot be read off the data surface. There is no entity masking, no explicit relational template, only a shared moral or causal motif (proverb) that must be inferred across narratives differing in domain, entities, and event structure. The fact that AAT reduces forgetting and improves generalization in this setting confirms that its inductive bias is not an artifact of the RCB's explicit relational structure. Jointly optimizing over instances and their structural abstractions, at whatever level that structure is defined, consistently biases learning toward the shared patterns that support both retention and transfer.

\subsection{Abstraction versus Instance Replay}

\begin{table*}[h]
\centering
\caption{Step-averaged online accuracy, forgetting, and cumulative accuracy on the RCB benchmark for Normal SFT, ER with an 100-sample buffer, and AAT across Qwen-1.5B, Qwen-3B, and SmolLM.}
\resizebox{\textwidth}{!}{ 

    \begin{tabular}{llccccccc}
    \toprule
    \textbf{Model} & \textbf{Method} & $\textbf{Acc}_{\textbf{online}}^{\textbf{unknown}}$ & $\textbf{Acc}_{\textbf{online}}^{\textbf{known}}$ & $\textbf{F}^{\textbf{unknown}}$ & $\textbf{F}^{\textbf{known}}$ & $\textbf{Acc}_{\textbf{cumul}}^{\textbf{unknown}}$ & $\textbf{Acc}_{\textbf{cumul}}^{\textbf{known}}$ & $\textbf{Acc}_{\textbf{cumul}}^{\textbf{all}}$ \\
    \midrule
     & Normal SFT & 0.4194 & 0.5487 & 0.0707 & 0.0767 & 0.4032 & 0.5330 & 0.4924 \\
    Qwen-1.5B & Buffer-100 & 0.4023 & \textbf{0.5656} & \textbf{0.0626} & \textbf{0.0635} & 0.3835 & \textbf{0.5515} & 0.4990 \\
    & AAT & \textbf{0.4220} & 0.5568 & 0.0643 & 0.0679 & \textbf{0.4086} & 0.5452 & \textbf{0.5025} \\ 
    
    \hdashline \noalign{\smallskip}
    & Normal SFT &0.4115 & 0.5730 & 0.1166 & 0.0496 & 0.3828 & 0.5552 & 0.5004 \\
    Qwen-3B & Buffer-100 & 0.4115 & \textbf{0.5818} & 0.1161 & \textbf{0.0483} & 0.3802 & \textbf{0.5721} & \textbf{0.5112} \\
    & AAT & \textbf{0.4141} & 0.5708 & \textbf{0.1035} & 0.0507 & \textbf{0.3880} & 0.5552 & 0.5021 \\

    \hdashline \noalign{\smallskip}
     & Normal SFT & 0.5708 & 0.7170 & 0.0300 & 0.0158 & 0.5565 & 0.7098 & 0.6619 \\
    SmolLM-1.7B & Buffer-100 & 0.6452 & 0.7166 & 0.0146 & \textbf{0.0099} & 0.6371 & 0.7152 & 0.6908 \\
    & AAT & \textbf{0.6613} & \textbf{0.7225} & \textbf{0.0060} & \textbf{0.0099} & \textbf{0.6559} & \textbf{0.7200} & \textbf{0.7000} \\ 
    \bottomrule
    \end{tabular}
}

\label{tab:main-table}
\end{table*}

Having established what abstraction does and why structural quality matters, we situate AAT relative to the instance-replay alternative. As described in \Cref{sec:baselines}, ER and AAT can be understood as two different instantiations of the same combined-loss template: both supplement the instance gradient with a stabilizing signal, but ER draws that signal from stored past instances while AAT derives it from the structural abstraction of the current one. The comparison therefore asks not which method is better in aggregate, but what each mechanism preserves and what it trades away.

Performance of all methods on RCB is shown in \Cref{tab:main-table}, including results for Qwen2.5-1.5B, SmolLM-1.7B, and Qwen2.5-3B. The pattern is consistent across all three models. Relative to normal SFT, AAT improves performance especially on unknown relations. Relative to ER, AAT produces a characteristic tradeoff with a 6.54\% gain on unknown edges at a modest 1.14\% cost on known edges for Qwen2.5-1.5B, with similar profiles for the 3B model. 
This tradeoff is not incidental, it is precisely what the cognitive framing predicts. A model biased toward shared relational structure should generalize better to unobserved relations at some cost to instance-level memorization. ER, stabilizing learning by revisiting past instances, shows the complementary profile, stronger known-edge retention without consistent gain on structural generalization. The two methods illuminate different points on the stability-generalization surface.

The SmolLM results show this most clearly. AAT improves cumulative accuracy over all edges by 5.76\% relative to normal SFT, reduces forgetting by 80\% on unknown edges and 37.34\% on known edges, and outperforms the replay baseline in cumulative accuracy while achieving the lowest forgetting overall. The consistency of this profile across two model families with different pretraining corpora and architectures suggests that the inductive bias introduced by abstraction reflects a structural property of the learning signal rather than an artifact of a particular model. Taken together, results across all three models show that abstraction and instance replay are complementary mechanisms that achieve stability through qualitatively different means. Abstraction preserves and generalizes relational structure; replay preserves episodic detail. 

Taken together, the results across all three analyses converge on a single conclusion: abstraction and instance replay are complementary mechanisms that achieve stability through qualitatively different means. Abstraction preserves and generalizes relational structure; replay preserves episodic detail. 

\section{Related Work}

\textbf{CL Benchmarks.} NLP continual learning benchmarks focuse on task-incremental settings with explicit boundaries \citep{aghajanyan2020better,wu2021curriculum,yadav-etal-2023-exploring,wang2023trace,zhao2025mllm,zhang-etal-2023-citb}, extending to domain or language shifts \citep{ke2023continual,castellucci-etal-2021-learning,mhamdi-etal-2023-cross,winata2023overcoming}. Parallel work explores OCL with non-stationary streams \citep{li-etal-2025-tic} or knowledge-graph updates \citep{jang2021towards,wu2023online}, but these benchmarks typically represent each update as a single fact or sentence and do not distinguish episodic retention from structural generalization. RCB and NAB are specifically designed to probe this distinction across different abstraction granularities, filling a gap left by existing benchmarks.

\textbf{CL Methods.} CL methods generally mitigate forgetting through regularization \citep{doi:10.1073/pnas.1611835114,chaudhry2018efficient} or ER \citep{pmlr-v199-ostapenko22a,Smith_2024_CVPR,NEURIPS2024_bb1e9f32,zhuo2023continual,Yan_2022_CVPR}. Replay is often effective \citep{van2020brain} but requires maintaining a memory buffer, scaling poorly in strictly online settings. Architectural strategies such as adapters and parameter expansion reduce interference \citep{gururangan-etal-2022-demix,mhamdi-etal-2023-cross,chen2024coin} but typically require task boundaries and continuous parameter growth. AAT occupies a distinct position: it introduces abstraction as a loss-level inductive bias at the moment of observation, targeting the instance-specific gradients that drive interference.

\textbf{Abstraction in Cognition.} Cognitive foundations of AAT draw on SMT, which proposes that analogical reasoning proceeds by aligning relational structure across domains while abstracting from surface attributes. Empirical work on analogical transfer shows that abstract relational encoding, rather than surface-level instance memorization, supports generalization across contexts \citep{gick1983schema,clement1994effects}. Studies in neuroscience further suggest that event representations organize around components shared across related experiences, supporting replay at multiple abstraction levels and enabling stable long-term learning \citep{masis2022schema,mattar2018prioritized}. In symbolic systems, structural abstraction has likewise been identified as a key signal for stability in non-stationary environments \citep{Sko1988,kuehne2000seql}.
AAT operationalizes these principles as a training signal, offering a computationally tractable bridge between cognitive accounts of schema formation and the challenge of non-stationary language model training.

\section{Conclusion}

Drawing on extensive cognitive evidence for the utility of abstract schemas, we introduced AAT to evaluate whether operationalizing structural abstraction as a training signal can yield the same functional benefits observed in human cognition.
Across two model families and two benchmarks probing different levels of abstraction, the answer is consistently: \textit{yes}. AAT reduces forgetting, improves generalization to unobserved relational structure, and reduces 
sensitivity to data order, a pattern of results that aligns with cognitive predictions for schema-congruent encoding and is not produced by instance-level 
replay alone. The characteristic tradeoff AAT produces, stronger structural generalization at modest cost to instance memorization, is not incidental but reflects the 
inductive bias itself, and maps directly onto the episodic-schematic distinction that cognitive accounts of memory draw. The gradation of performance across abstraction levels (random $<$ category $<$ entity 
masking) further confirms that structural coherence, not mere auxiliary supervision, governs the learning benefit; a result consistent with SMT's prediction that effective analogical generalization requires sufficient suppression of surface attributes to expose relational structure.

Beyond the practical implications for continual learning, these findings raise questions that we believe are productive for both communities. For machine learning: do the computational principles identified here, gradient-level suppression of surface variation and amplification of shared relational structure, generalize to larger models, implicit linguistic structure, and naturalistic data streams? For cognitive science: do these results offer traction on what makes structural abstraction a sufficient mechanism for stable generalization, and can this methodology be used to test other theories that have so far been studied primarily through behavioral experiments?
We view the intersection of these questions as a productive direction for future work, and offer the benchmarks and empirical results as a step toward a bidirectional dialogue between cognitive science and language model research in this domain.

\bibliographystyle{plainnat}
\bibliography{custom}


\appendix

\section{Appendix}

\subsection{Proof of Proposition 1} \label{appendix-proof}
\begin{proof}
By \textbf{A3}, the abstract loss contributes only a relational gradient:
$\nabla_\theta \mathcal{L}_i^{\mathrm{abs}} \propto W^\top \nabla\ell \cdot 
\nabla_\theta\phi_r^{(i)}$.
The combined gradient follows by linearity:
\[
\nabla_\theta \mathcal{L}_i^{\mathrm{total}} \propto 
W^\top \nabla \ell \cdot 
\Big[\nabla_\theta \phi_r^{(i)} + (1-\alpha)\nabla_\theta \phi_e^{(i)}\Big].
\]
By \textbf{A1}, the two components operate on disjoint parameter subsets, so:
\[
\mathrm{Var}\!\left(\nabla_\theta \mathcal{L}_i^{\mathrm{total}}\right) \propto 
 \sigma_r^2 + (1-\alpha)^2 \sigma_e^2.
\]
Since $\alpha \in (0,1)$, we have $(1-\alpha)^2 < 1$, so the entity variance term is strictly reduced relative to the instance-only case $\sigma_r^2 + \sigma_e^2$. The signal-to-noise improvement follows because $(1-\alpha) < 1$ strictly decreases the relational denominator.
\end{proof}

\subsection{Additional Ablation Studies}
\label{sec:appendix-ablation}
We conduct a series of ablation studies to analyze the sensitivity of AAT to its key hyperparameters and environmental factors. To study the abstraction loss weight, we train models with $\alpha \in \{0.1, 0.3, 0.5, 0.7, 0.9\}$ while fixing the local replay count to $n = 5$. Conversely, to evaluate the effect of local replay, we fix $\alpha = 0.5$ and vary $n$ from 1 to 5. In both cases we report cumulative accuracy at the final training step. To assess sensitivity to data order, we train each method on three random permutations of the data stream and report the mean and standard deviation of cumulative accuracy. In the online setting, performance variability across permutations reflects robustness to sequential interference and early sample bias. 
\subsubsection{Effect of AAT hyperparameters}

The effect of abstract loss weight ($\alpha$) in AAT is reported in \cref{fig:alph-effect}. For $\alpha < 0.5$, we observe substantial improvements on unknown edges, indicating enhanced inductive generalization, while gains on known edges remain limited. Performance peaks at $\alpha = 0.5$, where both known and unknown edges improve simultaneously, suggesting a favorable balance between instance-level supervision and abstraction-based regularization.

Beyond this point, performance degrades for both edge types. Although abstraction acts as a regularizer that encourages learning shared relational structure and reduces interference, assigning it excessive weight appears to suppress instance-specific learning. One plausible explanation is that heavily weighting the abstraction loss leads the model to overfit to repetitive abstract inputs, as multiple concrete instances share the same abstract representation. This result underscores the importance of balancing abstract and instance-level objectives rather than allowing abstraction to dominate the training signal.

\begin{figure*}[t]
    \centering
    \begin{tabular}{ccc}
        \begin{minipage}[c]{0.28\textwidth}
            \centering
            \includegraphics[width=\textwidth]{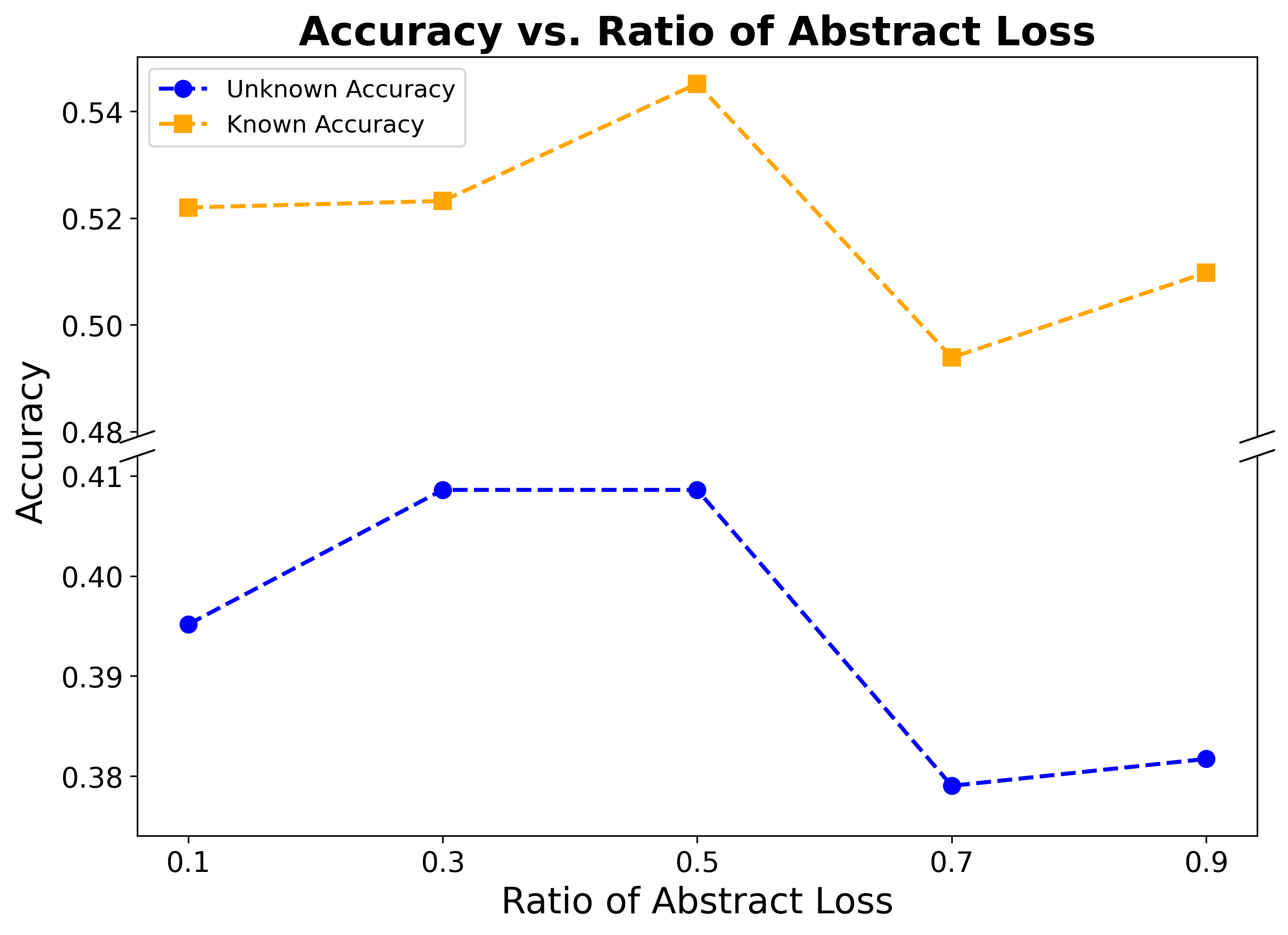}
            \caption{Effect of abstraction loss weight $\alpha$ on cumulative accuracy.}
            \label{fig:alph-effect}
        \end{minipage} & 
        
        \begin{minipage}[c]{0.37\textwidth}
            \centering
            \captionsetup{type=table} 
            \caption{Effect of data order on final cumulative accuracy}
            \label{tab:order-effect}

\resizebox{\linewidth}{!}{%
\renewcommand{\arraystretch}{1}
\begin{tabular}{lccc}
\toprule
\textbf{Method} & $\textbf{Acc}_{\textbf{cumul}}^{\textbf{unknown}}$ & $\textbf{Acc}_{\textbf{cumul}}^{\textbf{known}}$ & $\textbf{Acc}_{\textbf{cumul}}^{\textbf{all}}$ \\
\hline
\noalign{\smallskip}
Normal SFT & \makecell{0.3482 \\ $\pm$ 0.0392} & \makecell{0.5227 \\ $\pm$ 0.0121} & \makecell{0.4674 \\ $\pm$ 0.0197} \\
Buffer-100 & \makecell{0.3645 \\ $\pm$ 0.0189} & \makecell{0.5371 \\ $\pm$ 0.0161} & \makecell{0.4832 \\ $\pm$ 0.0169} \\
AAT & \makecell{0.3745 \\ $\pm$ 0.0265} & \makecell{0.5474 \\ $\pm$ 0.0070} & \makecell{0.4927 \\ $\pm$ 0.0071} \\
\bottomrule
\end{tabular}%
}
        \end{minipage} & 
        
        \begin{minipage}[c]{0.28\textwidth}
            \centering
            \includegraphics[width=\textwidth]{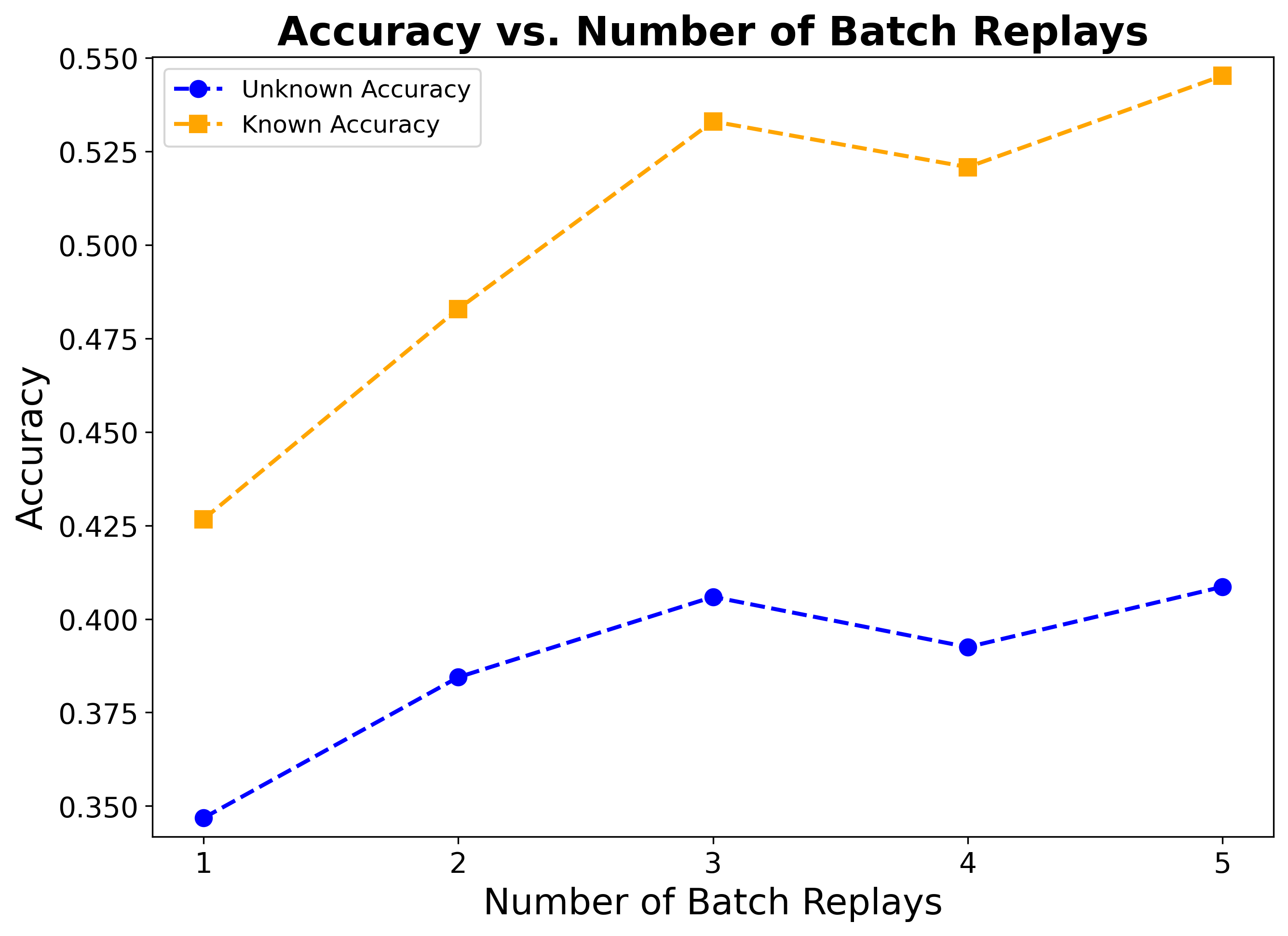}
            \caption{Effect of local replay count $n$ on cumulative accuracy.}
            \label{fig:n-effect}
        \end{minipage}
    \end{tabular}
\end{figure*}

\cref{fig:n-effect} shows the effect local replay count ($n$) on AAT. We observe a consistent improvement in performance on both known and unknown edges as the number of local replays increases. This trend indicates that additional local updates help stabilize learning in the strictly online setting, where each data point is observed only once. The gains from increasing $n$ are particularly pronounced for known edges, suggesting that local replay enables more effective consolidation of newly introduced entities and their relations. While there is a slight performance dip at $n = 4$, performance recovers at $n = 5$, indicating that the overall trend is robust. These results demonstrate that moderate local replay complements abstraction by improving instance-level learning without requiring long-term memory buffers.

\subsubsection{Effect of Data Order (Random Seeds)}

\begin{table*}[th]
\centering
\caption{Dataset statistics for RCB.}
\label{tab:rcb_domain_stats}
\begin{tabular}{lccccc}
\toprule
\textbf{Domain} & \textbf{Total} & \textbf{Derivable} &\textbf{ \% Deriv.} & \textbf{Known Edges} & \textbf{Unknown Edges} \\
\midrule
Genealogy   & 500 & 200 & 40.0\%  & 700 & 500 \\
Profession  & 180 & 180 & 100.0\% & 900 & 180 \\
Arts        & 115 & 50  & 43.5\%  & 280 & 115 \\
Science     & 100 & 25  & 25.0\%  & 125 & 100 \\
History     & 100 & 100 & 100.0\% & 200 & 100 \\
Music       & 100 & 25  & 25.0\%  & 125 & 100 \\
Geopolitics & 75  & 25  & 33.3\%  & 175 & 75  \\
Sports      & 75  & 75  & 100.0\% & 225 & 75  \\
\midrule
Total & 1245 & 680 & 54.6\% & 2730 & 1245 \\
\bottomrule
\end{tabular}
\end{table*}

\begin{table*}[th]
\centering
\caption{Performance across model sizes on the RCB and the NAB.}
\label{tab:benchmarking}
\begin{tabular}{lcccccc}
\toprule
\multirow{2}{*}{Model Size} 
& \multicolumn{3}{c}{RCB} 
& \multicolumn{3}{c}{NAB} \\
& Unknown & Known & Total 
& Unknown & Known & Total \\
\midrule
1.5B & 0.3469 & 0.4695 & 0.4311 & 0.5420 & 0.5103 & 0.5134 \\
3B   & 0.3020 & 0.4537 & 0.4062 & 0.5222 & 0.5266 & 0.5262 \\
7B   & 0.4357 & 0.6062 & 0.5528 & 0.5157 & 0.5198 & 0.5194 \\
\bottomrule
\end{tabular}

\end{table*}

We report the results of different training methods in \cref{tab:order-effect} showing training without abstraction exhibits the highest variance on unknown edges, indicating substantial sensitivity to data order when generalization relies solely on instance-level updates. ER reduces this variance, reflecting its ability to stabilize learning by revisiting past examples.

AAT achieves comparable or higher mean performance while exhibiting reduced variance relative to instance-only training, particularly on known edges and overall accuracy. Although variance on unknown edges remains non-negligible, the results suggest that abstraction mitigates order sensitivity by biasing learning toward relational structure shared across samples, rather than relying exclusively on the sequence in which specific instances are observed.

Overall, these findings indicate that abstraction provides a complementary mechanism to replay for stabilizing online learning, improving robustness to data order without requiring long-term storage of past examples.

\subsection{Implementation Details}\label{sec:imp-detail}

For ER, we employ a reservoir sampling strategy for buffer management. Updates to and sampling from the buffer occur at every training step, with the replay batch size set identical to the training batch size. We employ a warmup phase for the buffer, where replay sampling is disabled until the buffer reaches full capacity. We compare multiple ER configurations with different buffer sizes and local replay counts and use the best configuration in results presented in the main paper (see \Cref{tab:ER-table}). We utilize a learning rate of $5 \times 10^{-5}$ across all training methods. Experiments are conducted on three NVIDIA RTX A6000 GPUs with a per-device batch size of 4. Training on the largest dataset (NAB), with an evaluation interval of 10 steps, requires approximately 1 hour and 20 minutes, totaling about 24 hours across all experiments. To ensure reproducibility, the primary random seed is set to 42. To investigate the impact of data ordering, we employ two additional seeds: 25 and 16. Implementation is based on the Hugging Face ecosystem, specifically \texttt{transformers 4.56.1}, \texttt{torch 2.8.0}, and \texttt{accelerate 1.7.0}. For RCB evaluation, we compute accuracy using a confidence threshold applied to logits. The threshold $\tau$ is selected by collecting logits on a sample of the RCB data and choosing the elbow point in accuracy as a function of $\tau$. 

\begin{table*}[h]
\centering
\caption{Step-averaged online accuracy, forgetting, and cumulative accuracy on the RCB benchmark for ER configurations.}
\resizebox{\textwidth}{!}{ 

    \begin{tabular}{llccccccc}
    \toprule
    \textbf{Model} & \textbf{Method} & $\textbf{Acc}_{\textbf{online}}^{\textbf{unknown}}$ & $\textbf{Acc}_{\textbf{online}}^{\textbf{known}}$ & $\textbf{F}^{\textbf{unknown}}$ & $\textbf{F}^{\textbf{known}}$ & $\textbf{Acc}_{\textbf{cumul}}^{\textbf{unknown}}$ & $\textbf{Acc}_{\textbf{cumul}}^{\textbf{known}}$ & $\textbf{Acc}_{\textbf{cumul}}^{\textbf{all}}$ \\
    \midrule
    & Buffer-50-lp1 & 0.4176 & 0.5699 & 0.0792 & 0.0833 & 0.3934 & 0.5275 & 0.4856 \\
    Qwen-1.5B & Buffer-100-lp1 & 0.3772 & 0.5283 & 0.0950 & 0.0817 & 0.3450 & 0.5051 & 0.4551 \\
    & Buffer-50-lp5 & 0.4014 & 0.5679 & 0.0723 & 0.0689 & 0.3844 & 0.5515 & 0.4993 \\
     & Buffer-100-lp5 & 0.4023 & 0.5656 & 0.0626 & 0.0635 & 0.3835 & 0.5515 & 0.4990 \\
    
    \bottomrule
    \end{tabular}
}

\label{tab:ER-table}
\end{table*}

\subsection{Narrative Expansion}\label{appendix:epic_expansion}

To construct the NAB used in \Cref{sec:alt-abstractions}, we start from a collection of narratives paired with their corresponding proverbs from ePiC dataset \citep{ghosh-srivastava-2022-epic} and expand the dataset by generating 15 additional narratives for each proverb. The original dataset contains 250 proverbs, with 10 narratives associated with each proverb. This expansion increases the number of narrative instances per proverb while preserving the original set of proverbs. 
For narrative expansion, we prompted Qwen2.5-7B-Instruct \citep{qwen2.5}. Prompt design is informed by principles from SMT \citep{gentner1983structure}. For each proverb, an existing narrative is selected as a seed. Prompts discourage lexical overlap and object reuse, require a domain-level shift, and emphasize alignment at the level of abstract relational structure rather than surface form. In addition to generating narrative text, the model is prompted to identify aligned spans between the proverb and the narrative corresponding to shared relational elements. These span annotations are retained for analysis but are not used during training or evaluation. The full expansion prompt is shown in \Cref{fig:expansion_prompt}.

\subsection{Distractor Generation}\label{appendix:epic_distractor}

In NAB we convert each narrative instance into a pairwise continuation comparison task. Each narrative is split into a shared narrative part and two alternative continuations: (i) a \emph{correct ending} that is consistent with the proverb's abstract meaning, and (ii) a \emph{distractor ending} that is locally coherent but violates that abstract relation. For each narrative, we use sentence-level segmentation to break down the narrative into a shared narrative part and the final sentence of the narrative, and consider the final sentence as the correct ending. For distractor generation, we prompted Qwen2.5-7B-Instruct \citep{qwen2.5} with the proverb, the shared narrative part, and the correct ending, instructing the model to produce a single-sentence continuation that flows naturally from the narrative context but violates or ignores the proverb’s abstract meaning. Distractors are constrained to be plausible story continuations and are matched in length to the correct ending, ensuring that superficial cues such as fluency or stylistic coherence are insufficient for solving the task. The model additionally produces a brief explanation describing why the distractor does not align with the proverb, which is retained for analysis but not shown to the model during training or evaluation. The full distractor generation prompt is shown in \Cref{fig:distractor_prompt}. Each resulting instance defines a pairwise continuation comparison task in which the model is evaluated by comparing the log-likelihood assigned to each candidate ending given the same narrative part. Because both endings are locally coherent continuations of the same context, a higher likelihood therefore requires sensitivity to the implicit proverb-level abstraction rather than reliance on surface-level narrative consistency. 
\begin{table*}[h]
    \centering
    \caption{Random samples from the NAB: Proverbs, Narratives and Continuation Task.}
    \label{tab:nab-samples}
    \vspace{2mm}
    \footnotesize 
    \renewcommand{\arraystretch}{1.5} 
    \begin{tabularx}{\textwidth}{@{} l >{\RaggedRight\arraybackslash}X >{\RaggedRight\arraybackslash}X >{\RaggedRight\arraybackslash}X @{}}
        \toprule
        \textbf{Proverb} & \textbf{Narrative Context} & \textbf{Correct Ending} & \textbf{Distractor Ending} \\
        \midrule
        
        \textbf{Jam tomorrow...} & 
        The man believed his friend when he kept borrowing money from him, telling him he would pay him tomorrow. Every day he told him he would pay him soon. & 
        He kept giving him money, thinking of the future, but never saw the loans repaid. & 
        He finally got fed up with the constant requests and decided to save all his money for emergencies. \\
        
        \textbf{Keep your powder dry} & 
        Margo always knew to start any trip with a spare tire, emergency cash, and a car repair kit. She stayed calm and prepared for any catastrophe. & 
        She was yet to need any of this, but she always figured she'd be best prepared if catastrophe ever struck. & 
        With all the supplies in her car, Margo felt ready for anything and drove off, never expecting to encounter issues. \\
        
        \textbf{Laughter is the best medicine} & 
        My grandfather was ill and in a somber mood. I read him a funny speech about a plagiarism scandal, and his eyes lit up with fascination. & 
        The somber mood in the room dissipated like dew as we all wondered loudly, and mirthfully, how that was possible. & 
        The somber mood in the room dissipated as my grandfather's eyes lit up with curiosity about the scandal. \\
        
        \textbf{Procrastination is thief of time} & 
        My son put off an oil change that would have taken 20 minutes. 6 months later, he was sitting at the side of the road with a seized-up engine. & 
        Ultimately, we figured the delay resulted in over 29 hours of combined waiting, extra travel, and car searching. & 
        Fortunately, my son managed to get the engine fixed quickly by borrowing a friend's car, saving him from delays. \\
        
        \textbf{It takes two to tango} & 
        Jeff and Jane were getting into a fight. Jane pointed the finger, saying it was all Jeff's fault. Their mother knew better though. & 
        She knew that they both were causing trouble and Jeff and Jane were both punished. & 
        She quickly intervened and helped Jeff up off the floor, advising him to apologize to Jane. \\

        \bottomrule
    \end{tabularx}
\end{table*}

\subsection{Data Statistics and Samples}\label{appendix:datastats}
We introduced the Relational Cycle Benchmark (RCB) and the Narrative Abstraction Benchmark (NAB) to study the role of abstraction in non-stationary online learning. \Cref{tab:rcb_domain_stats} reports statistics for RCB, including the number of unique known and unknown edges per domain, as well as the fraction of unknown edges that are derivable within each domain. Representative samples from RCB and NAB are provided in \Cref{tab:rcb-sample} and \Cref{tab:nab-samples}. We additionally report results for different sizes of Qwen2.5-Instruct models on both benchmarks (see \Cref{tab:benchmarking}). On RCB, the 1.5B and 3B models exhibit relatively low performance, and increasing model size from 1.5B to 3B does not lead to substantial gains. In contrast, the 7B model achieves a marked improvement, suggesting that larger models are better able to retrieve relational information, including long-tailed facts. On NAB, performance remains largely unchanged across model sizes, with no noticeable scaling benefits. 
\begin{table*}[h]
    \centering
    \caption{Random samples from the RCB: Text Context and Knowledge Graph Edges}
    \label{tab:rcb-sample}
    \vspace{2mm}
    \footnotesize
    \renewcommand{\arraystretch}{1.7}
    \begin{tabularx}{\textwidth}{@{} 
        >{\RaggedRight\arraybackslash}p{5.5cm} 
        >{\RaggedRight\arraybackslash}p{4cm} 
        >{\RaggedRight\arraybackslash}p{4cm} 
        @{}}
        \toprule
        \textbf{Text Context} & \textbf{Known Edges (Sample)} & \textbf{Unknown Edge (Target)} \\
        \midrule
        
        Ljubljana is in Slovenia. Academy of Theatre, Radio, Film and Television is in Ljubljana. Polde Bibi\v{c} works for University of Ljubljana and was educated at the Academy. & 
        (Polde Bibi\v{c}, works for, Univ. of Ljubljana), (Polde Bibi\v{c}, educated at, Acad. of Theatre...), (Ljubljana, country, Slovenia) & 
        (Univ. of Ljubljana, based in country, Slovenia) \\
        
        Closing time is a dance film and has Mette Ingvartsen as a cast member. & 
        (Closing time, cast member, Mette Ingvartsen), (Closing time, genre, dance) & 
        (Mette Ingvartsen, field of work, dance) \\

        \bottomrule
    \end{tabularx}
\end{table*}

\begin{figure*}[th]
\centering
\begin{tcolorbox}[colback=gray!10, colframe=black, fontupper=\small]
\begin{verbatim}
You are an expert dataset creator for testing Analogical Reasoning. 
Your task is to generate a "Distractor Ending" for a story.

# GOAL
You will be given a Proverb, a Story Context, and the Correct Ending.
You must write a single sentence that:
1.  Flows naturally from the Story Context.
2.  Violates or Ignores the meaning of the Proverb. 

# CRITERIA FOR THE DISTRACTOR
- Plausible Storytelling: It should look like a real ending. Do not write gibberish.
- Thematic Mismatch: If the proverb says "haste makes waste" (bad outcome), 
your distractor should perhaps show haste leading to success.
- Length: Keep it roughly the same length as the Correct Ending.

# OUTPUT FORMAT
Strict JSON only:
{
    "distractor_ending": "[Your new sentence here]",
    "explanation": "[Briefly explain why this ending does not fit the proverb]"
}
\end{verbatim}
\end{tcolorbox}
\caption{Distractor generation prompt used for NAB.}
\label{fig:distractor_prompt}
\end{figure*}

\begin{figure*}[h]
\centering
\begin{tcolorbox}[colback=gray!10, colframe=black, fontupper=\small]
\begin{verbatim}
You are an expert Cognitive Scientist and NLP Researcher in Analogical 
Reasoning. Your task is to expand the ePiC dataset (Proverbs in Context) 
by applying Gentner's "Structure-Mapping Theory".

# TASK OVERVIEW
Generate a NEW narrative that is a structural analog to the proverb. 
Move from "Instance Learning" to "Abstraction" by focusing on 
higher-order relations rather than surface-level similarities.

# EXECUTION STEPS
1. Relational Extraction: Identify the core relational predicates.
2. Pragmatic Goal: Identify the "Purpose" (warn, comfort, criticize).
3. Domain Leap: Select a target domain different from the seed.
4. Systematic Alignment: Draft a narrative where the causal pattern 
   allows the relations to map to the proverb.
5. Evaluation: Review for "Object Overlap."
6. Span Discovery: Identify structural counterparts.

# CONSTRAINTS
- No Surface Overlap: No shared objects/themes with the proverb.
- Systematicity: Narrative must explain the "why" (causal chains).
- Span Alignment: Align up to 5 meaningful structural components.

# OUTPUT FORMAT (Strict JSON)
{
    "analysis": {
        "abstract_schema": "...",
        "domain_leap": "..."
    },
    "fields": {
        "quote": "{{quote_placeholder}}",
        "narrative": "...",
        "span_quote_1": "", "span_narrative_1": "",
        "span_quote_2": "", "span_narrative_2": "",
        "span_quote_3": "", "span_narrative_3": "",
        "span_quote_4": "", "span_narrative_4": "",
        "span_quote_5": "", "span_narrative_5": ""
    },
    "judge_scores": {
        "relational_similarity": "1-10",
        "surface_dissimilarity": "1-10",
        "rationale": "..."
    }
}
\end{verbatim}
\end{tcolorbox}
\caption{Narrative expansion prompt used for NAB.}
\label{fig:expansion_prompt}
\end{figure*}


\clearpage

\section*{NeurIPS Paper Checklist}

\begin{enumerate}

\item {\bf Claims}
    \item[] Question: Do the main claims made in the abstract and introduction accurately reflect the paper's contributions and scope?
    \item[] Answer: \answerYes{} 
    \item[] Justification: The abstract and introduction accurately summarize the research questions, methods, and empirical findings, and all stated claims are directly supported by results presented in the main text. 

\item {\bf Limitations}
    \item[] Question: Does the paper discuss the limitations of the work performed by the authors?
    \item[] Answer: \answerYes{} 
    \item[] Justification: Limitations are discussed in the conclusion and method sections, including the use of idealized environments and unaddressed empirical questions, along with implications for generalization and future work. 

\item {\bf Theory assumptions and proofs}
    \item[] Question: For each theoretical result, does the paper provide the full set of assumptions and a complete (and correct) proof?
    \item[] Answer: \answerYes{} 
    \item[] Justification: All assumptions and the main proposition are stated in the method section, and a complete proof is provided in the appendix with appropriate references. 

    \item {\bf Experimental result reproducibility}
    \item[] Question: Does the paper fully disclose all the information needed to reproduce the main experimental results of the paper to the extent that it affects the main claims and/or conclusions of the paper (regardless of whether the code and data are provided or not)?
    \item[] Answer: \answerYes{} 
    \item[] Justification: The appendix includes detailed implementation information, including hyperparameters, architectures, libraries, hardware specifications, and random seeds sufficient to reproduce the results. 

\item {\bf Open access to data and code}
    \item[] Question: Does the paper provide open access to the data and code, with sufficient instructions to faithfully reproduce the main experimental results, as described in supplemental material?
    \item[] Answer: \answerYes{} 
    \item[] Justification: An anonymized version of the code and datasets is provided in the supplementary material, along with instructions for reproducing the main experiments. 

\item {\bf Experimental setting/details}
    \item[] Question: Does the paper specify all the training and test details (e.g., data splits, hyperparameters, how they were chosen, type of optimizer) necessary to understand the results?
    \item[] Answer: \answerYes{} 
    \item[] Justification: The main paper describes the experimental design, data setup, and evaluation protocols, with full training details and hyperparameters provided in the appendix. 

\item {\bf Experiment statistical significance}
    \item[] Question: Does the paper report error bars suitably and correctly defined or other appropriate information about the statistical significance of the experiments?
    \item[] Answer: \answerYes{} 
    \item[] Justification: We report mean and standard deviation over multiple runs for the first experiment; due to computational constraints, not all experiments include multiple seeds, which we explicitly note.  

\item {\bf Experiments compute resources}
    \item[] Question: For each experiment, does the paper provide sufficient information on the computer resources (type of compute workers, memory, time of execution) needed to reproduce the experiments?
    \item[] Answer: \answerYes{} 
    \item[] Justification: We report the type of GPUs used, along with experiment runtime. 
    
\item {\bf Code of ethics}
    \item[] Question: Does the research conducted in the paper conform, in every respect, with the NeurIPS Code of Ethics \url{https://neurips.cc/public/EthicsGuidelines}?
    \item[] Answer: \answerYes{} 
    \item[] Justification: All datasets are derived from publicly available sources and used in accordance with their licenses; no human subjects are involved. 

\item {\bf Broader impacts}
    \item[] Question: Does the paper discuss both potential positive societal impacts and negative societal impacts of the work performed?
    \item[] Answer: \answerNA{} 
    \item[] Justification: This work is primarily methodological and does not present immediate real-world deployment; no direct societal impacts are identified. 
    
\item {\bf Safeguards}
    \item[] Question: Does the paper describe safeguards that have been put in place for responsible release of data or models that have a high risk for misuse (e.g., pre-trained language models, image generators, or scraped datasets)?
    \item[] Answer: \answerNA{} 
    \item[] Justification: The released assets consist of benchmark datasets derived from licensed public data and do not pose meaningful risks of misuse. 

\item {\bf Licenses for existing assets}
    \item[] Question: Are the creators or original owners of assets (e.g., code, data, models), used in the paper, properly credited and are the license and terms of use explicitly mentioned and properly respected?
    \item[] Answer: \answerYes{} 
    \item[] Justification: All external datasets, models, and libraries are properly cited, and their licenses and terms of use are respected. 

\item {\bf New assets}
    \item[] Question: Are new assets introduced in the paper well documented and is the documentation provided alongside the assets?
    \item[] Answer: \answerYes{} 
    \item[] Justification: We document the construction, structure, and intended use of the proposed benchmarks in dedicated sections and supplementary material.

\item {\bf Crowdsourcing and research with human subjects}
    \item[] Question: For crowdsourcing experiments and research with human subjects, does the paper include the full text of instructions given to participants and screenshots, if applicable, as well as details about compensation (if any)? 
    \item[] Answer: \answerNA{} 
    \item[] Justification: The work does not involve crowdsourcing or human subject research.

\item {\bf Institutional review board (IRB) approvals or equivalent for research with human subjects}
    \item[] Question: Does the paper describe potential risks incurred by study participants, whether such risks were disclosed to the subjects, and whether Institutional Review Board (IRB) approvals (or an equivalent approval/review based on the requirements of your country or institution) were obtained?
    \item[] Answer: \answerNA{} 
    \item[] Justification: The work does not involve human subjects and therefore does not require IRB approval. 

\item {\bf Declaration of LLM usage}
    \item[] Question: Does the paper describe the usage of LLMs if it is an important, original, or non-standard component of the core methods in this research? Note that if the LLM is used only for writing, editing, or formatting purposes and does \emph{not} impact the core methodology, scientific rigor, or originality of the research, declaration is not required.
    \item[] Answer: \answerYes{} 
    \item[] Justification: In creation of NAB we use an LLM to expand the data and create distractor endings for the downstream task. This is clearly discussed in the main text and the prompts and procedure are provided in appendix.

\end{enumerate}

\end{document}